\documentclass[12pt,transaction, draftclsnofoot,onecolumn]{IEEEtran}
\usepackage{amsfonts}
\usepackage{amssymb}
\usepackage{graphicx,epstopdf}
\usepackage{subfigure}
\usepackage{enumerate}
\usepackage{amsmath}
\usepackage{color}
\usepackage{amsthm}
\usepackage{amsmath}
\usepackage{algorithm}
\usepackage{algpseudocode}
\usepackage{cite}

\newcommand{\bp}{\begin{proof} \small }
\newcommand{\ep}{\end{proof} \normalsize}
\newcommand{\epx}{\end{proof} \small}
\newcommand{\bpa}{\begin{proofappx} \footnotesize }
\newcommand{\epa}{\end{proofappx} \small }
\newtheorem{theorem}{Theorem}

\newtheorem{lemma}{Lemma}

\newtheorem*{theorem*}{Theorem}
\newtheorem*{proposition*}{Proposition}
\newtheorem*{corollary*}{Corollary}
\newtheorem*{lemma*}{Lemma}
\newtheorem*{assumption*}{Assumption}
\newtheorem*{definition*}{Definition}
\newtheorem*{claim*}{Claim}

\newcommand{\be}{\begin{equation}}
\newcommand{\ee}{\end{equation}}
\newcommand{\bs}{\begin{subequations}}
\newcommand{\es}{\end{subequations}}
\newcommand{\bq}{\begin{eqnarray}}
\newcommand{\eq}{\end{eqnarray}}
\newcommand{\bqn}{\begin{eqnarray*}}
\newcommand{\eqn}{\end{eqnarray*}}

\newcommand{\ba}{\left[ \begin{array}}
\newcommand{\ea}{\\ \end{array} \right]}
\newcommand{\ben}{\begin{enumerate}}
\newcommand{\een}{\end{enumerate}}

\def\real{{\mathchoice%
{\hbox{\rm\setbox1=\hbox{I}\copy1\kern-.45\wd1 R}}
{\hbox{\rm\setbox1=\hbox{I}\copy1\kern-.45\wd1 R}}
{\hbox{\scriptsize\rm\setbox1=\hbox{I}\copy1\kern-.45\wd1 R}}
{\hbox{\scriptsize\rm\setbox1=\hbox{I}\copy1\kern-.45\wd1 R}}}}

\def\Zint{{\mathchoice{\setbox1=\hbox{\sf Z}\copy1\kern-.75\wd1\box1}
{\setbox1=\hbox{\sf Z}\copy1\kern-.75\wd1\box1}
{\setbox1=\hbox{\scriptsize\sf Z}\copy1\kern-.75\wd1\box1}
{\setbox1=\hbox{\scriptsize\sf Z}\copy1\kern-.75\wd1\box1}}}
\newcommand{\complex}{ \hbox{\rm C\kern-0.45em\rule[.07em]{.02em}{.58em}%
\kern 0.43em}}

\IEEEoverridecommandlockouts

\begin{document}
%
\title{Online Learning for Offloading and Autoscaling in Energy Harvesting Mobile Edge Computing}

\author{Jie~Xu,~\IEEEmembership{Member,~IEEE,}
        Lixing~Chen,~\IEEEmembership{Student~Member,~IEEE,}
        Shaolei~Ren,~\IEEEmembership{Member,~IEEE}
\thanks{J. Xu and L. Chen are with the Department of Electrical and Computer Engineering, University of Miami. Email: jiexu@miami.edu, lx.chen@miami.edu.}
\thanks{S. Ren is with the Department of Electrical and Computer Engineering, University of California, Riverside. Email: sren@ece.ucr.edu}
}


%


\maketitle

\begin{abstract}
Mobile edge computing (a.k.a. fog computing) has recently emerged to enable \emph{in-situ} processing of delay-sensitive applications at the edge of mobile networks. Providing grid power supply in support of mobile edge computing, however, is costly and even infeasible (in certain rugged or under-developed areas), thus mandating on-site renewable energy as a major or even sole power supply in increasingly many scenarios. Nonetheless, the high intermittency and unpredictability of renewable energy make it very challenging to deliver a high quality of service to users in energy harvesting mobile edge computing systems.
In this paper, we address the challenge of incorporating renewables into mobile edge computing and propose an efficient reinforcement learning-based resource management algorithm, which learns on-the-fly the optimal policy of dynamic workload offloading (to the centralized cloud) and edge server provisioning to minimize the long-term system cost (including both service delay and operational cost). Our online learning algorithm uses a decomposition of the (offline) value iteration and (online) reinforcement learning, thus achieving a significant improvement of learning rate and run-time performance when compared to standard reinforcement learning algorithms such as Q-learning. We prove the convergence of the proposed algorithm and analytically show that the learned policy has a simple monotone structure amenable to practical implementation. Our simulation results validate the efficacy of our algorithm, which significantly improves the edge computing performance compared to fixed or myopic optimization schemes and conventional reinforcement learning algorithms.


\end{abstract}


%
\IEEEpeerreviewmaketitle

\section{Introduction}

In the era of mobile computing and Internet of Things, a tremendous amount of data is generated from massively distributed sources, requiring timely processing to extract its maximum value. Further, many emerging applications, such as mobile gaming and augmented reality, are delay sensitive and have resulted in an increasingly high computing demand that frequently exceeds what mobile devices can deliver. Although cloud computing enables convenient access to a centralized pool of configurable computing resources, moving all the distributed data and computing-intensive applications to clouds (which are often physically located in remote mega-scale data centers) is simply out of the question, as it would not only pose an extremely heavy burden on today's already-congested  backbone networks \cite{rivera2014gartner} but also result in (sometimes intolerable) large transmission  latencies that degrade the quality of service \cite{beck2014mobile,vaquero2014finding,shi2016edge}.

As a remedy to the above limitations, mobile edge computing (MEC) \cite{beck2014mobile,vaquero2014finding,shi2016edge} (a.k.a., fog computing \cite{chiang2016fog}) has recently emerged to enable \emph{in-situ} processing of (some) workloads locally at the network edge without moving them to the cloud. In MEC, network edge devices, such as base stations, access points and routers, are endowed with cloud-like computing and storage capabilities to serve users' requests as a substitute of clouds, while significantly reducing the transmission latency as they are placed in close proximity to end users and data sources. In this paper, we consider base station as the default edge device and refer to the combination of an edge device and the associated edge servers as an {\it edge system}.

Effective operation of MEC is contingent upon efficient power provisioning for the edge system. However, providing reliable and stable grid power supply in remote areas and hazardous locations can be extremely costly and even infeasible since construction and operation of transmission lines are often prohibitive, and grid-tied servers can violate environmental quality regulations in rural areas that are ecologically sensitive \cite{li2015towards}. For instance, in many developing countries, the majority of base stations have to be powered by continuously operating diesel generators because the electric grid is too unreliable \cite{TaoHan_UNCC_LoadBalancingRAN_HybriedEnergy_ToN}. In view of the significant carbon footprint of grid power as well as soaring electricity prices, off-grid renewable energy harvested from ambient vibrations, heat, wind and/or solar radiation is embraced as a major or even sole power supply for edge systems in the field, thanks to the recent advancements of energy harvesting techniques \cite{sudevalayam2011energy,ulukus2015energy}.

Despite the clear advantages, the high intermittency and unpredictability of renewable energy creates tremendous new challenges for fully reaping the benefits of MEC. Although batteries are often installed as an energy buffer, the computing capacity of an edge system is still significantly limited at any moment in time. As a result, although processing computation tasks at the edge reduces the transmission latency, a considerable processing time may occur when little power supply is available. This gives rise to an important trade-off between transmission delay and processing delay, which is jointly determined by the edge system's offloading policy (i.e. how much workload is offloaded to the cloud) and autoscaling policy (i.e. how many servers are dynamically provisioned or activated). The problem is further complicated due to the temporal correlation --- provisioning more servers and processing more workloads at the edge system in the current time means that fewer servers can be provisioned and fewer workloads can be processed locally in the future due to the limited and time-varying renewable energy supply. Figure \ref{system} illustrates the considered system architecture.

\begin{figure}
  \centering
  \includegraphics[width=0.45\textwidth]{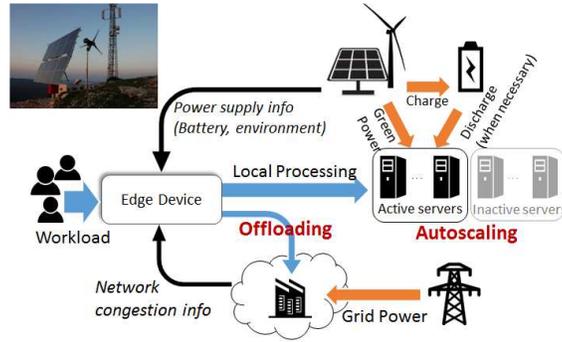}\\
  \caption{Architecture of a renewable-powered edge computing system. The photo shows a solar- and wind-powered base station deployed by Alcatel Lucent in Turkey. (Source: http://www.cellular-news.com/tags/solar/wind-power/)}\label{system}
\end{figure}

In this paper, we address the challenge of incorporating renewable energy into MEC and propose an efficient reinforcement learning-based resource management algorithm, which learns on-the-fly the optimal policy of dynamic workload offloading (to the centralized cloud) and edge server provisioning to minimize the long-term system cost (including both service delay and operational cost). Our main contributions are summarized as follows:
\begin{itemize}
  \item We formulate the joint offloading and edge server provisioning problem as a Markov decision process (MDP) by taking into account various unique aspects of the considered energy harvesting MEC system. The offloading and edge server provisioning decisions are jointly made according to the information of computation workload, core network congestion state, available battery power and anticipated renewable power arrival. By formulating the MDP, the edge system resource management is carried out in a foresighted way by taking future system dynamics into account, thereby optimizing the long-term system performance.
  \item We develop a novel post-decision state (PDS) based learning algorithm that learns the optimal joint offloading and autoscaling policy on-the-fly. It is well-known that MDP suffers from the so-called ``curse of dimensionality'' problem when the state space is large \cite{sutton1998reinforcement}. The proposed PDS-based learning algorithm exploits the special structure of state transitions of the considered energy harvesting MEC system to conquer this challenge, thereby significantly improving both the learning convergence speed and the run-time performance compared with conventional online reinforcement learning algorithms such as Q-learning \cite{sutton1998reinforcement}. The key to achieving this performance improvement is a decomposition of the (offline) value iteration and (online) reinforcement learning that allows many components of the algorithm to be learned in a batch manner.
  \item We prove the convergence of the proposed PDS-based learning algorithm and analytically characterize the structure of the resulting optimal policy. The optimal policy is proven to have a simple monotone structure: the power demand for the optimal joint offloading and autoscaling policy is non-decreasing in the amount of available battery power. This result enables easy implementation of the proposed algorithm in practical MEC applications.
  \item Extensive simulations are carried out to verify our analytical results and evaluate the performance of the proposed algorithm. The results confirm that our method can significantly improve the performance of the energy harvesting MEC system.
\end{itemize}

The rest of this paper is organized as follows. Section II discusses related works. Section III describes the system model. Section IV formulates the MDP problem. Section V develops the PDS-based learning algorithm. Section VI proves the convergence of the proposed algorithm and characterizes the structure of the optimal policy. Section VII evaluates the proposed method via systematic simulations. Section VIII concludes the paper.

\section{Related Work}
Mobile edge computing (MEC) has received an increasing amount of attention in recent years. The concept of MEC was proposed in 2014 as a new platform that provides IT and cloud-computing capabilities within the radio access network in close proximity to mobile subscribers \cite{patel2014mobile}. Initially, MEC refers to the use of BSs for offloading computation tasks from mobile devices. Recently, the definition of edge devices gets broader, encompassing any devices that have computing resources along the path between data sources and cloud data centers \cite{shi2016edge}. Fog computing \cite{chiang2016fog} is a related concept that refers to the same computing paradigm. The areas of Fog computing and MEC are overlapping and the terminologies are frequently used interchangeably. There exist significant disparities between MEC and mobile cloud computing (MCC). Compared with MCC, MEC has the advantages of achieving lower latency, saving energy, supporting context-aware computing, and enhancing privacy and security for mobile applications \cite{shi2016edge}. A central theme of many prior studies is offloading policy on the \emph{user} side, i.e., what/when/how to offload a user's workload from its device to the edge system or cloud (see \cite{mao2017mobile} and references therein). Depending on the type of tasks, offloading can be either binary or partial. Our work focuses on edge-side offloading and autoscaling, and hence is complementary to these studies on user-side offloading.

MEC servers are small-scale data centers and consume substantially less energy than the conventional cloud mega-scale data center \cite{mao2017mobile}. However, as MEC servers become more widely deployed, the system-wide energy consumption becomes a big concern. Therefore, innovative techniques for achieving green MEC is in much need. Off-grid renewable energy, such as solar radiation and wind energy, has recently emerged as a viable and promising power source for various IT systems thanks to the recent advancement of energy harvesting techniques \cite{sudevalayam2011energy,ulukus2015energy}. Compared with traditional grid energy which is normally generated by coal-fired power plants, employing renewable energy significantly reduces the amount of carbon emission. Moreover, the use of renewable energy sources eliminates the need of human intervention, which is difficult if not impossible for certain types of application scenarios where the devices are hard and dangerous to reach. However, designing green MEC powered by renewable energy is much more challenging compared to green communication systems \cite{chia2014data,USC_DynamicBaseStation_Switching_On_Off_TWireless_2013,Krishnamachari_USC_Energy_efficient_wireless_basestation_CommunicationsMagazine_2011} or green data center networks \cite{LinWiermanAndrewThereska}\cite{TaoHan_UNCC_LoadBalancingRAN_HybriedEnergy_ToN,ChaoLi_iSwitch_ISCA_2012_Li:2012:ICO:2337159.2337218,Goiri:2011:GSE:2063384.2063411} since the radio and computation resources have to be jointly managed, whereas prior research typically only considers one of the two decisions. For example, autoscaling (a.k.a., right-sizing) in data centers \cite{LinWiermanAndrewThereska} dynamically controls the number of active servers, but the control knob of offloading to the cloud is not available in the context of data centers. For energy harvesting mobile devices, a dynamic computation offloading policy was proposed in \cite{mao2016dynamic} using Lyapunov optimization techniques \cite{neely2010stochastic} based on both information of the wireless channel and energy. The focus of the present paper is on energy harvesting MEC systems and our solution is based on reinforcement learning.

Another study relevant to our work is \cite{deng2015towards}, which also studies workload allocation/offloading in a cloud-fog computing system. However, unlike our energy harvesting edge system, this paper considers a grid-powered system and focuses on a one-shot static optimization without addressing the temporal correlation among the offloading decisions across time (due to intermittent renewable energy and limited battery capacity). The present paper develops a foresighted resource management algorithm for energy harvesting MEC, which can operate optimally in time-varying and unknown environments by formulating a Markov decision process problem. To cope with unknown time-varying system dynamics, our learning algorithm employs a decomposition of (offline) value iteration and (online) reinforcement learning based on factoring the system dynamics into an \textit{a priori} known and an \textit{a priori} unknown component. A key advantage of our proposed algorithm is that it exploits the partial information of the edge computing system and the structure of the resource management problem, and thus it converges much faster than conventional reinforcement learning algorithms such as Q-learning \cite{sutton1998reinforcement}.

To the best knowledge of the authors, the conference version of this paper \cite{xu2016online} was the first to study resource management for energy harvesting MEC servers (see related discussions in a recent comprehensive survey paper \cite{mao2017mobile}). The present paper extents our findings in \cite{xu2016online}. Specifically, in addition to developing an efficient learning algorithm for the optimal offloading and autoscaling policy, we analytically characterize the structure of the learned optimal resource management policy and carry out extensive simulations to evaluate its performance.

\section{System Model}

As a major deployment method of mobile edge computing \cite{patel2014mobile}, we consider an edge system consisting of a base station and a set of edge servers, which are physically co-located and share the same power supply in the cell site.
\begin{table}
\center
\caption{Main Notations and Their Meanings}
\begin{tabular}{|c|c|}
  \hline
  symbol & meaning \\
  \hline
  $\lambda(t)$ & total workload arrival rate \\
  $\mu(t)$ & amount of locally processed workload\\
  $m(t)$ & number of active servers\\
  $c_{wi}(t)$ & wireless access and transmission delay cost \\
  $c_{lo}(t)$ & local processing delay cost \\
  $c_{off}(t)$ & offloading delay cost \\
  $c_{delay}(t)$ & total delay cost\\
  $c_{bak}(t)$ & backup power supply cost\\
  $d_{op}(t)$ & edge system operation power consumption\\
  $d_{com}(t)$ & edge system computing power consumption \\
  $d(t)$ & total energy consumption\\
  $g(t)$ & harvested green energy \\
  $e(t)$ & environment state \\
  $b(t)$ & battery level \\
  $h(t)$ & backhaul network congestion state\\
  $s(t)$ & system state\\
  $\tilde{s}(t)$ & post-decision system state\\
  $C(t)$ & normal value function\\
  $V(t)$ & post-decision value function\\
  \hline
\end{tabular}
\end{table}

\subsection{Workload model}
We consider a discrete-time model by dividing the operating period into time slots of equal length indexed by $t=0,1,...$, each of which has a duration that matches the timescale at which the edge device can adjust its computing capacity (i.e. number of active servers). We use $x \in \mathcal{L}$ to represent a location coordinate in the service area $\mathcal{L}$. Let $\lambda(x,t)$ represent the workload arrival rate in location $x$, and $\theta(x, t)$ be the wireless transmission rate between the base station and location $x$. Thus $\lambda(t) = \sum_{x\in\mathcal{L}} \lambda(x,t) \in [0, \lambda_{max}]$ is the total workload arrival rate at the edge system, where $\lambda_{max}$ is the maximum possible arrival rate. The system decides the amount of workload $\mu(t) \leq \lambda(t)$ that will be processed locally. The remaining workload $\nu(t) \triangleq \lambda(t) - \mu(t)$ will be offloaded to the cloud for processing. The edge system also decides at the beginning of the time slot the number of active servers, denoted by $m(t) \in [0, M] \triangleq \mathcal{M}$. These servers are used to process the local workload $\mu(t)$. Since changing the number of servers during job execution are difficult and in many cases impossible, we only allow determining the number of servers at the beginning of each time slot but not within the slot.

\subsection{Delay cost model}
The average utilization of the base station is $\rho(t) = \sum_x \lambda(x,t)/\theta(x,t)$,
which results in a total wireless access and transmission delay of $c_{wi}(t) = \sum_{x} \lambda(x,t)/[\theta(x,t)(1-\rho(t))]$ by following the literature and
modeling the base
station as a queueing system \cite{USC_DynamicBaseStation_Switching_On_Off_TWireless_2013}. Next we model the workload processing delay incurred at the edge servers.

For the local processed workload, the delay cost $c_{lo}(t)$ is mainly processing delay due to the limited computing capacity at the local edge servers. The transmission delay from the edge device to the local servers is negligible due to physical co-location. To quantify the delay performance of services, such as average delay and tail delay (e.g. 95th-percentile latency), without restricting our model to any particular performance metric, we use the general notion of $c_{lo}(m(t), \mu(t))$ to represent the delay performance of interest during time slot $t$. As a concrete example, we can model the service process at a server instance as an M/G/1 queue and use the average response time (multiplied by the arrival rate) to represent the delay cost, which can be expressed as $c_{lo}(m(t), \mu(t)) = \frac{\mu(t)}{m(t)\cdot \kappa-\mu(t)}$, where $\kappa$ is the service rate of each server.

For the offloaded workload, the delay cost $c_{off}(t)$ is mainly transmission delay due to network round trip time (RTT), which varies depending on the network congestion state. For modeling simplicity, the service delay at the cloud side is also absorbed
into the network congestion state. Thus, we model the network congestion state, denoted by $h(t)$, as an exogenous parameter and express it in terms of the RTT (plus cloud service delay) for simplicity. The delay cost is thus $c_{off}(h(t), \lambda(t), \mu(t)) = (\lambda(t) - \mu(t))h(t)$. The total delay cost is therefore
\begin{align}
&c_{delay}(h(t), \lambda(t), m(t), \mu(t)) \nonumber\\
= &c_{lo}(m(t), \mu(t)) + c_{off}(h(t), \lambda(t), \mu(t)) + c_{wi}(\lambda(t))
\end{align}

\subsection{Power model}
We interchangeably use power and energy, since energy consumption during each time slot is the product of (average) power and the duration of each time slot that is held constant in our model. The total power demand of the edge system in a time slot consists of two parts: first, basic operation and transmission power demand by edge devices (base station in our study); and second, computing power demand by
edge servers. The first part is independent of the offloading or the autoscaling policy, which is modeled as $d_{op}(\lambda(t)) = d_{sta} + d_{dyn}(\lambda(t))$
where $d_{sta}$ is the static power consumption and $d_{dyn}(\lambda(t))$ is the dynamic power consumption depending on the amount of total workload. The computing power demand depends on the number of active servers as well as the locally processed workload. We use a generic function $d_{com}(m(t), \mu(t))$, which is increasing in $m(t)$ and $\mu(t)$, to denote the computing power demand. The total power demand in time slot $t$ is therefore
\begin{align}
d(\lambda(t), m(t), \mu(t)) = d_{op}(\lambda(t)) + d_{com}(m(t), \mu(t))
\end{align}

To model the uncertainty of the green power supply, we assume that the green power budget, denoted by $g(t)$, is realized after the offloading and autoscaling decisions are made. Therefore, the decisions cannot utilize the exact information of $g(t)$. However, we assume that there is an environment state $e(t)$ which the system can observe and it encodes valuable information of how much green energy budget is anticipated in the current time slot. For instance, daytime in a good weather usually implies high solar power budget. Specifically, we model $g(t)$ as an i.i.d. random variable given $e(t)$, which obeys a conditional probability distribution $P_g(g(t)|e(t))$. Note that the environment state $e(t)$ itself may not be i.i.d.

\subsection{Battery model}
Batteries are used to balance the power supply and demand. In a solar+wind system, photovoltaic modules and wind turbines can combine their output to power the edge system and charge the batteries \cite{li2015towards}. When their combined efforts are insufficient, batteries take over to ensure steady operation of the edge system. We denote the battery state at the beginning of time slot $t$ by $b(t) \in [0, B] \triangleq \mathcal{B}$ (in units of power) where $B$ is the battery capacity. For system protection reasons, the battery unit has to be disconnected from the load once its terminal voltage is below a certain threshold for charging. We map $b(t) = 0$ to this threshold voltage to ensure basic operation of the system.

\begin{figure}
  \centering
  \includegraphics[width=0.4\textwidth]{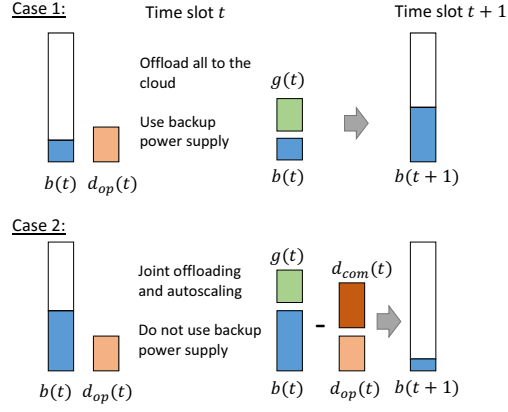}\\
  \caption{Battery state dynamics. Case 1: current battery cannot support basic operation and thus, backup power supply is invoked. Case 2: current battery can support basic operation. }\label{battery}
\end{figure}

Since green power budget is unpredictable and hence unknown at the beginning of time slot $t$, the edge system uses a conservative policy which satisfies $d_{com}(m(t), \mu(t)) \leq \max\{b(t) - d_{op}(\lambda(t)),0\}$ to avoid activating backup power supply by making offloading and autoscaling decisions.
\begin{itemize}
  \item When $d_{op}(\lambda(t)) \geq b(t)$, $d_{com}(\lambda(t), m(t), \mu(t))$ must be zero, which means that the edge system offloads all workload to the cloud if the existing battery level cannot even support the basic operation and transmission in the current slot. Moreover, the backup power supply (e.g. diesel generator) will be used to maintain basic operation for the slot. The cost due to activating the backup power supply is $c_{bak}(t) = \phi\cdot d_{op}(\lambda(t))$ where $\phi > 0$ is a large constant representing the large cost due to using the backup power supply. The next time slot battery state then evolves to $b(t+1) = b(t) + g(t)$.
  \item When $d_{op}(\lambda(t)) \leq b(t)$, the edge system may process part of the workload $\mu(t) \leq \lambda(t)$ at the local servers, but the power demand must satisfy $d_{com}(\lambda(t), m(t), \mu(t)) \leq b(t) - d_{op}(\lambda(t))$. Depending on the realized green power $g(t)$ and the computing power demand $d_{com}(\lambda(t), m(t), \mu(t))$, the battery is recharged or discharged accordingly:
\begin{itemize}
  \item  If $g(t) \geq d(\lambda(t), m(t), \mu(t))$, then the surplus $g(t) - d(\lambda(t), m(t), \mu(t))$ is stored in the battery until reaching its capacity $B$:
      \begin{equation}
      b(t+1) = \max\{b(t) + g(t) - d(\lambda(t), m(t), \mu(t)), B\}
      \end{equation}
  \item If $g(t) < d(\lambda(t), m(t), \mu(t))$, then the battery has to be discharged to cover the energy deficit $d(\lambda(t), m(t), \mu(t)) - g(t)$.
      \begin{equation}
      b(t+1) = b(t) + g(t) - d(\lambda(t), m(t), \mu(t))
      \end{equation}
      For simplicity, we will assume that there is no power loss either in recharging or discharging the batteries, noting that this can be easily generalized. We also assume that the batteries are not leaky. We model the battery depreciation cost in a time slot, denoted by $c_{battery}(t)$, using the amount of discharged power in this time slot since the lifetime discharging is often limited. Specifically,
        \begin{align*}
        c_{battery}(t) = \omega\cdot\max\{d(\lambda(t), m(t), \mu(t)) - g(t),0\}
        \end{align*}
      where $\omega > 0$ is the normalized unit depreciation cost.
\end{itemize}
\end{itemize}

\section{Problem Formulation}
In this section, we formulate the dynamic offloading and autoscaling problem as an online learning problem, in order to minimize the system cost. The system state is described by a tuple $s(t) \triangleq (\lambda(t), e(t), h(t), b(t))$, which is observable at the beginning of each time slot. Among the four state elements, the workload arrival rate $\lambda(t)$, the environment state $e(t)$, the backbone network state $h(t)$ are exogenous states which are independent of the offloading and autoscaling actions while the battery state $b(t)$ evolves according the to offloading and autoscaling actions as well as the renewable power realization. To make the stochastic control problem tractable, they are assumed to have finite value spaces and $\lambda(t), e(t), h(t)$ evolve as finite-state Markov chains. Specifically, let $P_\lambda(\lambda(t+1)|\lambda(t))$, $P_e(e(t+1)|e(t))$ and $P_h(h(t+1)|h(t))$ denote the transition matrices for $\lambda(t)$, $e(t)$ and $h(t)$, respectively. Similar assumptions have been made in existing literature, e.g. \cite{GuenterJainWilliams,zhang2014structure}. Importantly, all these probability distributions are unknown {\it a priori} to the edge system.

The stochastic control problem now can be cast into an MDP, which consists of four elements: the state space $\mathcal{S}$, the action space $\mathcal{A}$, the state transition probabilities $P_s(s(t+1)|s(t), a(t)), \forall s,s'\in\mathcal{S}, a\in \mathcal{A}$, and the cost function $c(s,a), \forall s, a$. We have already defined the state space. Next we introduce the other elements as follows.

\textbf{Actions}. Although the actual actions taken by the edge system are $\nu(t)$ (offloading) and $m(t)$ (autoscaling) in each time slot $t$, we will consider an intermediate action in the MDP formulation, which is the computing power demand in each time slot $t$, denoted by $a(t) \in \mathcal{A}$ where $\mathcal{A}$ is a finite value space $[0,1,....,B]$. We will see in a moment how to determine the optimal offloading and autoscaling actions based on this. As mentioned before, to maintain basic operation in the worst case, we require that $a(t) \leq \max\{b(t) - d_{op}(\lambda(t)), 0\}$. Thus, this condition determines the feasible action set in each time slot.

\textbf{State transitions}. Given the current state $s(t)$, the computing power demand $a(t)$ and the realized green power budget $g(t)$, the buffer state in the next time slot is
\begin{align}\label{buffer}
&b(t+1) = [b(t) + g(t)]_0^B, \textrm{if}~d_{op}(\lambda(t)) > b(t)\\
&b(t+1) = [b(t) - d_{op}(\lambda(t)) - a(t) + g(t)]_0^B, \textrm{otherwise} \nonumber
\end{align}
where $[\cdot]_0^B$ denotes $\max\{\min\{\cdot, B\}, 0\}$. The system then evolves into the next time slot $t+1$ with the new state $s(t+1)$. The transition probability from $s(t)$ to $s(t+1)$, given  $a(t)$, can be expressed as follows
\begin{align}
&P(s(t+1)|s(t),a(t)) \nonumber\\
=&P_\lambda(\lambda(t+1)|\lambda(t))P_e(e(t+1)|e(t))P_h(h(t+1)|h(t)) \nonumber\\
\times &\sum\limits_{g(t)}P_g(g(t)|e(t))\textbf{1}\{\zeta(t)\}
\end{align}
where $\textbf{1}\{\cdot\}$ is the indicator function and $\zeta(t)$ denotes the event defined by \eqref{buffer}. Notice that the state transition only depends on  $a(t)$ but not  the offloading or the autoscaling action. This is why we can focus on the computing power demand action $a(t)$ for the foresighted optimization problem.

\textbf{Cost function}. The total system cost is the sum of the delay cost, the battery depreciation cost and the backup power supply cost. If $d_{op}(\lambda(t)) > b(t)$, then the cost is simply
\begin{align}\label{oneslotcost1}
\tilde{c}(t) = c_{delay}(h(t), \lambda(t), 0, 0) + c_{bak}(\lambda(t))
\end{align}
since we must have $m(t) = 0$ and $\mu(t) = 0$. Otherwise, the realized cost given the realized green power budget $g(t)$ is
\begin{align*}
\tilde{c}(t) =  c_{delay}(h(t), \lambda(t), m(t), \mu(t))
+  \omega\cdot [a(t) - g(t)]_0^\infty
\end{align*}
Since the state transition does not depend on $\mu(t)$ or $m(t)$, they can be  optimized given $s(t)$ and $a(t)$ by solving the following myopic optimization problem
\begin{align}\label{joint}
\min_{\mu, m}~~ c_{delay}(h, \lambda, m, \mu)~~\textrm{s.t.}~~d(m, \mu) = a
\end{align}
Let $m^*(s, a)$ and $\mu^*(s, a)$ denote the optimal solution and $c^*_{delay}(s, a)$ the optimal value given $s$ and $a$. Therefore, the minimum cost in time slot $t$ given $s$ and $a$ is
\begin{align}\label{oneslotcost2}
\tilde{c}(t) = c^*_{delay}(s(t), a(t))+ \omega\cdot [a(t) - g(t)]_0^\infty
\end{align}
The expected cost is thus
\begin{align*}
c(s(t), a(t)) = c^*_{delay}(s(t), a(t)) +  E_{g(t)|e(t)}\omega\cdot [a(t) - g(t)]_0^\infty
\end{align*}

\textbf{Policy}. The edge system's computing power demand policy (which implies the joint offloading and autoscaling policy) in the MDP is a mapping $\pi: \Lambda\times \mathcal{E}\times \mathcal{H} \times \mathcal{B} \to \mathcal{A}$. We focus on optimizing the policy to minimize the edge system's expected long-term cost, which is defined as the expectation of the discounted sum of the edge device's one-slot cost: $C^\pi(s(0)) = \mathbb{E}\left(\sum\limits_{t=0}^\infty \delta^t c(s(t),a(t)) | s(0)\right)$ where $\delta < 1$ is a constant discount factor, which models the fact that a higher weight is put on the current cost than the future cost. The expectation is taken over the distribution of the green power budget, the workload arrival, the environment state and the network congestion state. It is known that in MDP, this problem is equivalent to the following optimization
\begin{align}
\min_\pi C^\pi(s), \forall s\in \mathcal{S}
\end{align}
Let $C^*(s)$ be the optimal discounted sum cost starting with state $s$. It is well-known that $\pi^*$ and $C^*(s)$ can be obtained by recursively solving the following set of Bellman equations
\begin{align}\label{valuefunction}
C^*(s) = \min_{a\in \mathcal{A}}\left(c(s,a) + \delta \sum\limits_{s'\in\mathcal{S}} P(s'|s,a)C^*(s')\right), \forall s
\end{align}
In the next section, we solve this problem using the idea of dynamic programming and online learning.

\section{Post-Decision State Based Online Learning}
If all the probability distributions were known a priori, then the optimal policy  could be solved using traditional algorithms for solving Bellman equations, e.g. the value iteration and the policy iteration \cite{sutton1998reinforcement}, in an offline manner. In the considered problem, all these probability distributions are unknown a priori and hence, these algorithms are not feasible. In this section, we propose an online reinforcement learning algorithm to derive the optimal policy $\pi^*$ on-the-fly. Our solution is based on the idea of post-decision state (PDS), which exploits the partially known information about the system dynamics and allows the edge system to integrate this information into its learning process to speed up learning. Compared with conventional online reinforcement learning algorithms, e.g. Q-learning, the proposed PDS based learning algorithm significantly improves its convergence speed and run-time performance.

In this rest of this section, we first define PDS, and then describe the proposed algorithm. Finally, we prove the convergence of the proposed algorithm.

\subsection{Post-Decision State}
\begin{figure}
  \centering
  \includegraphics[width=0.4\textwidth]{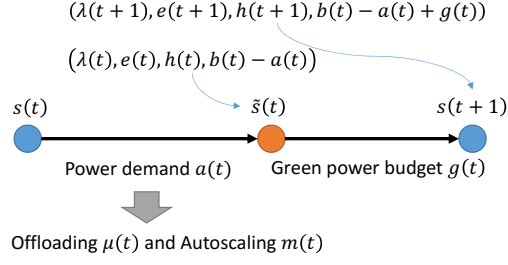}\\
  \caption{Illustration of Post-Decision State}\label{PDS}
\end{figure}

We first introduce the notion of PDS, which is the most critical idea of our proposed algorithm. In our problem, PDS is the intermediate system state after the edge system takes the computing power demand action $a(t)$ but before the green power budget $g(t)$ is realized. Figure \ref{PDS} illustrates the relationship between a normal state $s(t)$ and its PDS $\tilde{s}(t)$. Specifically, the PDS in time slot $t$, denoted by $\tilde{s}(t) \triangleq (\tilde{\lambda}(t), \tilde{e}(t), \tilde{h}(t), \tilde{b}(t))$, is defined as
\begin{align}\label{pds1}
\tilde{\lambda}(t) = \lambda(t),~~\tilde{e}(t) = e(t), ~~\tilde{h}(t) = h(t)
\end{align}
and
\begin{equation}\label{pds2}
\tilde{b}(t) = \left\{
\begin{array}{l}
b(t), \textrm{if}~d_{op}(\lambda(t))>b(t)\\
\max\{b(t) -d_{op}(\lambda(t)) - a(t),0\}, \textrm{otherwise}
\end{array}\right.
\end{equation}

As we can see, the post-decision workload state $\tilde{\lambda}(t)$, post-decision environment state $\tilde{e}(t)$ and post-decision network congestion state $\tilde{h}(t)$ remain the same because the computing power demand action $a(t)$ does not have a direct impact on these elements of the system state. The only element of the system state that may change is the battery state $b(t)$. However, it is important to notice that the post-decision battery state $\tilde{b}(t)$ is only a virtual state but not the real battery state. Given the definition of PDS, we further define the post-decision value function $V^*(\tilde{s})$ as follows:
\begin{align}\label{PD_valuefunction}
V^*(\tilde{s}) = \sum\limits_{s'\in \mathcal{S}}P(s'|\tilde{s})C^*(s'), \forall s
\end{align}
where the transition $P(s'|\tilde{s})$ between PDS and the next system state is now independent of the action:
\begin{align}
\tilde{P}(s|\tilde{s})
= &P_\lambda(\lambda|\tilde{\lambda})P_e(e|\tilde{e})P_h(h|\tilde{h}) \nonumber\\
&\times \sum_{g} P_g(g|\tilde{e})\textbf{1}\{b = \min\{\tilde{b} + g, B\}\}
\end{align}
For the ease of exposition, we refer to $s$ as the ``normal'' state and $C^*(s)$ as the ``normal'' value (cost) function, in order to differentiate with their post-decision counterparts.

By comparing \eqref{valuefunction} and \eqref{PD_valuefunction}, it is obvious that there is a deterministic mapping from the normal value function $C^*(s)$ to the post-state value function $V^*(\tilde{s})$ as follows by
\begin{align}\label{UV}
C^*(s) = \min_{a\in\mathcal{A}}(c(s,a) + \delta V^*(\tilde{s}))
\end{align}
The above equation shows that the normal value function $C^*(s)$ in each time slot is obtained from the corresponding post-decision value function $V^*(\tilde{s})$ in the same time slot, where $\tilde{s} = (\lambda, e, h, b)$ if $d_{op}(\lambda)) > b$ and $\tilde{s} = (\lambda, e, h, \max\{b-d_{op}(\lambda)-a, 0\})$ if $d_{op}(\lambda) \leq b$, by performing the minimization over the action $a$.

The advantages of using the PDS and post-decision value function are summarized as follows.

(1) In the normal state based Bellman's equation set \eqref{valuefunction}, the expectation over the possible workload arrival $\lambda$, the environment state $e$, network congestion state $h$, and green power budget $g$ has to be performed before the minimization over the possible energy demand actions $a$. Therefore, performing the minimization requires the knowledge of these dynamics. In contrast, in the PDS based Bellman equations \eqref{PD_valuefunction}, the expectation operation is separated from the minimization operation. If we can learn and approximate the post-decision value function $V^*(\tilde{s})$, then the minimization can be solved without any prior knowledge of the system dynamics.

(2) Given the energy demand action $a$, the PDS decomposes the system dynamics into an a priori unknown component, i.e. $\lambda$, $e$, $h$ and $g$ whose evolution is independent of $a$, and an a priori known component, i.e. the battery state evolution is partially determined by $a$. Importantly, $\lambda$, $e$, $h$ and $g$ are also independent of the battery state $b$. This fact enables us to develop a batch update scheme on the post-decision value functions, which can significantly improve the convergence speed of the proposed PDS based reinforcement learning.

\subsection{The algorithm}
The algorithm maintains and updates a set of variables in each time slot. These variables are
\begin{itemize}
\item The one slot cost estimate $\hat{c}^t(s,a), \forall (s, a) \in \mathcal{S}\times \mathcal{A}$.
\item The post-decision value function estimate $\hat{V}^t(\tilde{s}), \forall \tilde{s} \in \mathcal{S}$.
\item The normal value function estimates $\hat{C}^t(s),\forall s\in \mathcal{S}$.
\end{itemize}
The superscript $t$ is used to denote the estimates at the beginning of the time slot $t$. If these estimates are accurate, i.e. $\hat{c}^t(s,a) = c(s,a)$, $\hat{V}^t(\tilde{s}) = V^*(\tilde{s})$ and $\hat{C}^t(s) = C^*(s)$, then the optimal power demand policy is readily obtained by solving \eqref{UV}. Our goal is to learn these variables over time using the realizations of the system states and costs. The algorithm works as follows: (In each time slot $t$)

\textbf{Step 1}: Determine the empirically optimal computing power demand by solving
\begin{align}\label{actionselection}
a(t) = \min_a (\hat{c}^t(s(t),a) + \delta \hat{V}^t(\tilde{s}(t)))
\end{align}
using the current estimates $\hat{c}^t(s(t),a)$ and $\hat{V}^t(\tilde{s}(t))$, where for each $a$, $\tilde{s}(t)$ represents the corresponding PDS. Given this power demand, the corresponding optimal offloading and autoscaling actions are determined as $\mu(t) = \mu^*(s(t), a(t))$ and $m(t) = m^*(s(t), a(t))$ based on the solution of \eqref{joint}.

After the green power $g(t)$ is harvested and hence the current slot cost $\tilde{c}(t)$ is realized according to \eqref{oneslotcost2}, the battery state evolves to $b(t+1)$ according to \eqref{buffer}.

Steps 2 through 4 update the estimates.

\textbf{Step 2}: Batch update $\hat{c}^t(s, a)$ for any action $a$ and any state $s = (\lambda,e, h, b)$ such that $e$ is the same as the current slot environment state $e(t)$ using the realized green power budget $g(t)$ according to
\begin{align}\label{batchcost}
\hat{c}^{t+1}(s,a) = (1-\rho^t) \hat{c}^{t}(s,a) + \rho^t c(s, a, g(t))
\end{align}
where $\rho^t$ is the learning rate factor that satisfies $\sum_{t=0}^\infty \rho^t = \infty$ and $\sum\limits_{t=0}^\infty (\rho^t)^2 < \infty$. For all other action-state pair, $\hat{c}^{t+1}(s,a) = \hat{c}^t(s,a)$. We can do this batch update because the green power budget $g(t)$ depends only on the environment state $e(t)$ but not on other states or actions.

\textbf{Step 3}: Batch update the normal value function estimate for any state $s = (\lambda, e, h, b)$ such that $e = e(t)$ according to
\begin{align}\label{batchvalue}
\hat{C}^{t+1}(s) = \min_{a\in\mathcal{A}}(\hat{c}^{t+1}(s,a) + \delta \hat{V}^t(\tilde{s}))
\end{align}
The normal value function estimates for the remaining states are unchanged. The reason why we can do this batch update is the same as that in Step 2.

\textbf{Step 4}: Batch update the post-decision value function estimate for any $\tilde{s}\in \tilde{\mathcal{S}}$ such that $\tilde{\lambda} = \tilde{\lambda}(t)$, $\tilde{e} =\tilde{e}(t)$ and $\tilde{h} = \tilde{h}(t)$ according to
\begin{align}\label{batchpdvalue}
\hat{V}^{t+1}(\tilde{s}) = (1-\alpha^t) \hat{V}^{t}(\tilde{s}) + \alpha^t \hat{C}^{t+1}(s)
\end{align}
where $s = (\lambda, e, h, b)$ satisfies $\lambda = \lambda(t+1)$, $e = e(t+1)$, $h = h(t+1)$ and $b = \min\{\tilde{b} + g(t), B\}$. In this way, we update not only the currently visited PDS $\tilde{s}(t)$ but all PDS with common $\tilde{\lambda}(t)$, $\tilde{e}(t)$ and $\tilde{h}(t)$. This is because the temporal transition of $\lambda, e, h$ is independent of of the battery state $b$ and the green power budget realization follows the same distribution since the environment state $e$ is the same for these states.

\begin{algorithm}
\caption{Online Learning for Joint Offloading and Autoscaling}
\begin{algorithmic}[1]
\State \textbf{Initialize}: $\hat{c}^0(s,a) = 0, \forall (s, a)$, $\hat{V}^0(\tilde{s}) = 0, \forall \tilde{s}$, $\hat{C}^0(s) = 0, \forall s$.
\For{every time slot $t$}
\State Observe current state $s(t)$
\State Determine power demand $a(t)$ by solving \eqref{actionselection}
\State Determine offloading and autoscaling actions $\mu(t)$ and $m(t)$ by solving \eqref{joint}
\State Compute the post-decision state $\tilde{s}^t$ according to \eqref{pds1} and \eqref{pds2}
\State (The green power budget $g(t)$ is realized.)
\State Batch update $\hat{c}^t(s, a)$ according to \eqref{batchcost}
\State Batch update $\hat{C}^t(s)$ according to \eqref{batchvalue}
\State Batch update $\hat{V}^t(\tilde(s))$ according to \eqref{batchpdvalue}
\EndFor
\end{algorithmic}
\end{algorithm}

\section{Algorithm Analysis}
\subsection{Convergence of the PDS learning algorithm}
We first prove the convergence of our algorithm.
\begin{theorem}
The PDS based online learning algorithm converges to the optimal post-decision value function $V^*(\tilde{s}), \forall \tilde{s}$ when the sequence of learning rates $\rho^t$ satisfies $\sum_{t=0}^\infty \rho^t = \infty$ and $\sum\limits_{t=0}^\infty (\rho^t)^2 < \infty$.
\end{theorem}
\begin{proof}
The proof follows \cite{zhang2014structure}. For each PDS $\tilde{s}$, we define a function on its value function as follows:
\begin{align}
F_{\tilde{s}}(V) = \min_{a\in\mathcal{A}}(c(s,a) + \delta V(\tilde{s}))
\end{align}
where $s$ and $a$ are such that, $\lambda = \tilde{\lambda}$, $e = \tilde{e}$, $h = \tilde{h}$ and $b = \tilde{b} - a$. Thus, for any value of $V(\tilde{s})$, $F_{\tilde{s}}(V)$  maps to a real number. Based on this, we define $F: \mathbb{R}^{|\mathcal{S}|} \to \mathbb{R}^{|\mathcal{S}|}$ be a mapping which collects $F_{\tilde{s}}$ for all $\tilde{s} \in \mathcal{S}$. It is proven in \cite{borkar2000ode} that the convergence of our proposed algorithm is equivalent to the convergence of the associated ordinary differential equation (O.D.E.):
\begin{align}
\dot{V} = F(V) - V
\end{align}
Since the map $F: \mathbb{R}^{|\mathcal{S}|} \to \mathbb{R}^{|\mathcal{S}|}$ is a maximum norm $\delta$-contraction, the asymptotic stability of the unique equilibrium point of the above O.D.E. is guaranteed \cite{bertsekas1995dynamic}. This unique equilibrium point corresponds to the optimal post-decision value function $V^*(\tilde{s}), \forall \tilde{s} \in \mathcal{S}$.
\end{proof}

Because $C^*(s), \forall s$ is a deterministic function of $V^*(\tilde{s}), \forall \tilde{s}$, it is straightforward that the PDS based online learning algorithm also converges to $C^*(s), \forall s$. Therefore, we prove that the edge system is able to learn the optimal power demand policy and hence the optimal offloading and autoscaling policies using the proposed algorithm.

\subsection{Structure of the Optimal Policy}
Next, we characterize the structure of the optimal policy. First, we show that the one-slot cost function is convex in the power demand action.
\begin{lemma}
Assume that both $c_{lo}(m, \mu)$ and $d_{com}(m, \mu)$ are jointly convex in $(m, \mu)$ for any given $s$, then the one-slot cost function $c(s, a)$ is convex in $a$ for any given $s$.
\end{lemma}
\begin{proof}
Recall that $c(s, a) = c^*_{delay}(s, a) +  E_{g|e}\omega\cdot \max\{a - g, 0\}$. Since $a-g$ is linear in $a$, and the maximum of convex functions is still convex, it is obvious that $\max\{a - g, 0\}$ is convex in $a$. Since the expectation is just a weighted sum of convex functions, $E_{g|e}\omega\cdot \max\{a - g, 0\}$ is also convex in $a$. Now, if we can prove $c^*_{delay}(s, a)$ is convex in $a$, then the lemma is proved. Recall that $c^*_{delay}(s, a)$ is the solution to
\begin{align}
\min_{\mu, m}~~ c_{delay}(h, \lambda, m, \mu)~~\textrm{s.t.}~~d(m, \mu) = a
\end{align}
for the given $s$. Since $c_{delay}(h, \lambda, m, \mu) = c_{lo}(m, \mu) + c_{off}(\mu(t)) + c_{wi}$, $c_{lo}(m, \mu)$ is jointly convex in $(m, \mu)$ and $c_{off}(\mu(t))$ is linear in $\mu$, $c_{delay}(h, \lambda, m, \mu)$ is also jointly convex in $(m, \mu)$. Similarly, since $d(m, \mu) = d_{op} + d_{com}(m, \mu)$ and $d_{com}(m, \mu)$ is jointly convex in $(m, \mu)$, $d(m, \mu)$ is also jointly convex in $(m, \mu)$.

Consider two power demand actions $a_1, a_2$, let $(m_1^*, \mu_1^*)$ and $(m_2^*, \mu_2^*)$ be the corresponding optimal joint offloading and autoscaling actions. Clearly, we should have the constraint be binding for the optimal solution, i.e. $d(m_1^*, \mu_1^*) = a_1^*$ and $d(m_2^*, \mu_2^*) = a_2^*$. Now, we have, $\forall \lambda \in (0,1)$,
\begin{align}
&\lambda c^*_{delay}(a_1) + (1-\lambda) \lambda c^*_{delay}(a_2)\nonumber\\
=& \lambda c_{delay}(m_1^*, \mu_1^*) + (1-\lambda) \lambda c_{delay}(m_2^*, \mu_2^*)\\
\geq &c_{delay}(\lambda(m_1^*, \mu_1^*) + (1-\lambda)(m_2^*, \mu_2^*)) \triangleq c_{delay}(m_\lambda, \mu_\lambda)\nonumber
\end{align}
where we define $(m_\lambda, \mu_\lambda) \triangleq \lambda(m_1^*, \mu_1^*) + (1-\lambda)(m_2^*, \mu_2^*)$. Let $\tilde{a} = d(m_\lambda, \mu_\lambda)$ be the corresponding required power demand. Further, we let $(m^*, \mu^*)$ be the optimal offloading and autoscaling action for this $\tilde{a}$. Clearly, $c_{delay}(m_\lambda, \mu_\lambda) \geq c_{delay}(m^*, \mu^*)$. Due to the convexity of $d(m,\mu)$, we have
\begin{align}
&\tilde{a} = d(m_\lambda, \mu_\lambda) = d(\lambda(m_1^*, \mu_1^*) + (1-\lambda)(m_2^*, \mu_2^*))\nonumber\\
\leq &\lambda d(m_1^*, \mu_1^*) + (1-\lambda) d(m_2^*, \mu_2^*) \\
=& \lambda a_1 + (1-\lambda)a_2 \nonumber
\triangleq a_\lambda
\end{align}
Therefore, $c_{delay}(m^*, \mu^*) \geq c_{delay}(m^*_\lambda, \mu^*_\lambda)$, where $(m^*_\lambda, \mu^*_\lambda)$ are the optimal offloading and autoscaling actions for $a_\lambda$. Therefore, we have
\begin{align}
&\lambda c^*_{delay}(a_1) + (1-\lambda) \lambda c^*_{delay}(a_2) \nonumber\\
\geq& c_{delay}(m^*_\lambda, \mu^*_\lambda) = c^*_{delay}(\lambda a_1 + (1-\lambda)a_2)
\end{align}
This completes the proof of this lemma.
\end{proof}

Lemma 2 characterizes the shape of the optimal value function and post-state value function.
\begin{lemma}
Assume that $c(s, a)$ is convex in $a$ for any given $s$, then both $V^*(s)$ and $C^*(s)$ are non-increasing and convex in $b$ for any given $\lambda, e, h$.
\end{lemma}
\begin{proof}
We first prove $C^*(s)$ is non-increasing. The optimal value functions satisfy
\begin{align}
C^*(s) = \min_{a\in\mathcal{A}}\left(c(s,a) + \delta \sum_{s'\in\mathcal{S}}P(s'|s,a)C^*(s')\right)
\end{align}
Consider two states $s$ and $s'$ that differs only in $b$ and $b'$, and assume $b < b'$. Let $a^*$ be the optimal action for $s$. Now consider an action $a' = b' - b + a^* > a^*$ for state $s'$. It is obvious that the second term on the right-hand side is identical in both cases since the transitions of $\lambda, e, h$ are independent of the battery state, and because $b'-a' = b-a^*$, the battery state transition is also the same. Because $c(s,a)$ is non-increasing in $a$, we have $c(s', a') \leq c(s, a^*)$. Therefore, by choosing $a'$ for $s'$, we have $C(s', a') \leq C^*(s)$. Realizing $C^*(s') \leq C(s', a')$ due to the minimization operation, we have $C^*(s') \leq C^*(s)$, thus proving $C^*(s)$ is non-increasing.

Next, we prove that $C^*(s)$ is convex by induction. The optimal value functions $C^*(s),\forall s$ can be solved by the value iteration algorithm
\begin{align}
C_{n+1}(s) = \min_{a\in\mathcal{A}}\left(c(s,a) + \delta \sum_{s'\in\mathcal{S}}P(s'|s,a)C_{n}(s')\right)
\end{align}
where the subscript $n$ and $n+1$ represent the $n$-th iteration and $(n+1)$-th iteration. It is well known that valuation iteration converges to the optimal solution, i.e. as $n\to \infty$, $C_{n}(s) \to C^*(s), \forall s$, starting from any initial value function $C_0(s), \forall s$. We initialize $C_0(s)$ to be convex in $b$ for any given $\lambda, e, h$. It is easy to see that $V^*(s)$ is also non-increasing since it is simply a weighted average of a bunch of value functions.

Suppose $C_n(s), \forall s$ are convex in $b$ for any given $\lambda, e, h$. Consider two battery states $b < b'$ and let the corresponding optimal action be $a^*$ and $a'^*$. Then we have
\begin{align}
C_{n+1}(b) = c(a^*) + \delta V_n(b-a^*) \label{c1}\\
C_{n+1}(b') = c(a'^*) + \delta V_n(b'-a'^*) \label{c2}
\end{align}
In the above equations, we omitted the state elements $\lambda, e, h$. Since we have assumed that $C_n(b), \forall s$ are convex in $b$. In addition, $\hat{b} = \min\{b+g, B\}$ is a concave function in $b$. Thus, by applying the results of composition (i.e. if $f$ is concave and $g$ is convex and non-increasing, then $h(x) = g(f(x))$ is convex.), we have $C_n(\min\{b+g, B\}), \forall s$ also convex. Therefore $V_n(b)$ is also convex in $b$ since it is a weighted sum of convex functions.

Now, combining \eqref{c1} and \eqref{c2} and using the convexity of $V_n(b)$ we have
\begin{align}
&\lambda C_{n+1}(b) + (1 - \lambda) C_{n+1}(b')\nonumber\\
=& \lambda c(a^*) + (1-\lambda) c(a'^*) \nonumber\\
&+ \delta[\lambda V_n(b-a^*) + (1-\lambda) V_n(b'-a'^*)]\nonumber\\
\geq & c(a_\lambda) + \delta V_n(b_\lambda - a_\lambda) \geq C_{n+1}(b_{\lambda})
\end{align}
where $b_\lambda = \lambda b + (1-\lambda) b'$ and $a_\lambda = \lambda a^* + (1-\lambda) a'^*$. This proves that $C_{n+1}(b)$ is also convex in $b$.
\end{proof}

Now, we are ready to prove the structural result of the optimal power demand policy.
\begin{theorem}
Assume that both $c_{lo}(m, \mu)$ and $d_{com}(m, \mu)$ are jointly convex in $(m, \mu)$ for any given $s$, then the optimal power demand policy is monotonically non-decreasing in $b$ for any given $\lambda, e, h$. That is, $\forall s, s'$ such that $\lambda = \lambda'$, $e = e'$, $h = h'$ and $b \leq b'$, then we have $\pi^*(s) \leq \pi(s')$.
\end{theorem}
\begin{proof}
We aim to prove that $C(b, a)$ is subadditive in the battery state $b$ and the power demand action $a$ for any given $\lambda, e, h$. This is to prove that $\forall b \leq b'$ and $a \leq a'$, we have
\begin{align}\label{subadditive}
C(b, a') - C(b, a) \geq C(b', a') - C(b', a)
\end{align}
If the above is true, then we can apply \cite{puterman2014markov} (Section 4.7) to show that $\pi(s)$ is non-decreasing in $b$.

Equation \eqref{subadditive} is equivalent to
\begin{align}
V^*(b - a') - V^*(b - a) \geq V^*(b' - a') - V^*(b' - a)
\end{align}
Let $b - a = \hat{b}$, $b'- a = \hat{b}'$ and $\Delta = a' - a$. Then the above becomes
\begin{align}
V^*(\hat{b} - \Delta) - V^*(\hat{b}) \geq V^*(\hat{b}' - \Delta) - V^*(\hat{b}')
\end{align}
which is true due to the fact that $V^*(b)$ is non-increasing and convex when $c_{lo}(m, \mu)$ and $d_{com}(m, \mu)$ are jointly convex in $(m, \mu)$ for any given $s$ (by Lemma 1 and Lemma 2). This completes the proof.
\end{proof}

\section{Simulation}
\subsection{Simulation Setup}
We consider each time slot as 15 minutes. The workload arrival space is set as $\Lambda=$\{10 units/sec, 20 units/sec, ..., 100 units/sec\}. The network congestion space is $\mathcal{H}=$\{20 ms/unit, 30 ms/unit, ..., 60 ms/unit\}. The environment state space is $\mathcal{E}=$\{Low, Medium, High\}. For each environment state, the green power will be realized according to a normal distrinution with different means: $g(t|e=\text{Low})\sim \mathcal{N}(200W,10^2)$, $g(t|e=\text{Medium})\sim \mathcal{N}(400W,10^2)$, $g(t|e=\text{High})\sim \mathcal{N}(600W,10^2)$. The battery capacity is set as $B=$2 kWh. The base station static power consumption is $d_{sta}=$ 300W. The maximum number of activated edge server is $M = 15$. The power consumption of each edge server is 150W. The maximum service rate of each edge server is 20 units/sec. Other important parameters are set as follows:  the normalized unit depreciation cost $\omega=0.01$, the cost coefficient of backup power supply $\phi=0.15$.

The proposed PDS-based learning algorithm is compared with three benchmark schemes:
\begin{itemize}
    \item \textbf{Q-learning}\cite{sutton1998reinforcement}: Q-learning is a famous model-free reinforcement learning technique for solving MDP problems. It has been proven that for any finite MDP problem, Q-learning eventually finds an optimal policy.
    \item \textbf{Myopic optimization}: this scheme ignores the temporal correlation between the system states and the decisions, and minimizes the cost function given the state in the current time slot by utilizing all available battery energy.
    \item \textbf{Fixed power}: this scheme uses a fixed computation power (whenever possible) for edge computing in each time slot.
\end{itemize}
\subsection{Run-time Performance Comparison}
\begin{figure}
	\centering\includegraphics[width=0.45\textwidth]{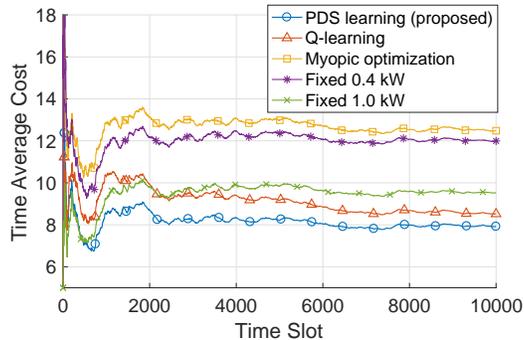}
	\caption{Run-time performance comparison}\label{time_average_cost}
\end{figure}

Figure \ref{time_average_cost} compares the run-time performance of our scheme with the three benchmark schemes for 10000 time slots. As can be seen, the proposed PDS-based learning algorithm incurs a significantly lower cost than all benchmark schemes.
\begin{itemize}
  \item Myopic optimization incurs a large time-average cost since it ignores the temporal correlation of decision making and frequently is forced to activate the backup power in the subsequent time slots.
  \item The fixed power scheme has tremendously different performance depending on the fixed value used, which implies that it is sensitive to system parameters. In Figure \ref{time_average_cost}, two fixed values (1.0kW and 0.4kW) are shown for illustrative purposes, where 1.0kW is the best fixed value found by our extensive simulations. Since the system dynamics is unknown a priori and may change over time, using a fixed computing power scheme may cause significant performance loss.
  \item The performance of Q-learning is much worse than our proposed PDS learning since it converges very slowly due to the large state space. Although there is a trend of declining in the time average cost, even after 10000 time slots, there is still a considerable performance gap compared with our scheme. On the other hand, our proposed PDS scheme converges very quickly.
\end{itemize}

\subsection{Learned Optimal Policy}
\begin{figure}
	\centering\includegraphics[width=0.45\textwidth]{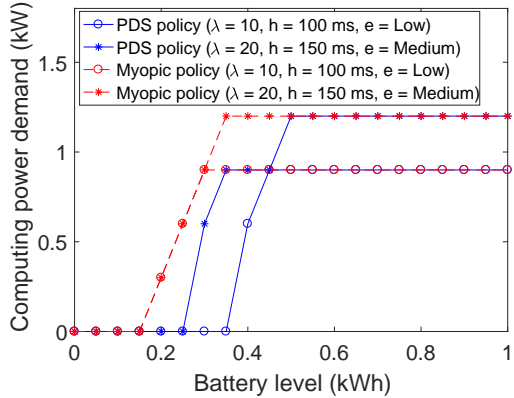}
	\caption{Computing power demand policy}\label{battery_computation}
\end{figure}

Figure \ref{battery_computation} further explains why the proposed algorithm outperforms the myopic solution by showing the learned optimal policy. When the workload demand is low and the network is not congested, the policy learned by the proposed algorithm tends to be conservative in using local computing power if the battery level is low. In this way, more power can be saved for future when the workload is high and the network congestion state degrades, thereby reducing the long-term system cost. On the other hand, the myopic policy ignores this temporal correlation, it activates local servers to process workload even if the battery level is not so high. As a result, even though it achieves slight improvement in the current slot, it wastes power for potentially reducing significant cost in the future. Figure \ref{battery_computation} also validates our theoretical results in Theorem 2 on the structure of the optimal policy: The optimal power demand increases in the battery level.

\subsection{Battery State Distribution}
\begin{figure}
	\centering	
	\subfigure[Battery state distribution]{\label{battery_dis_bar}
		\includegraphics[width=0.45\textwidth]{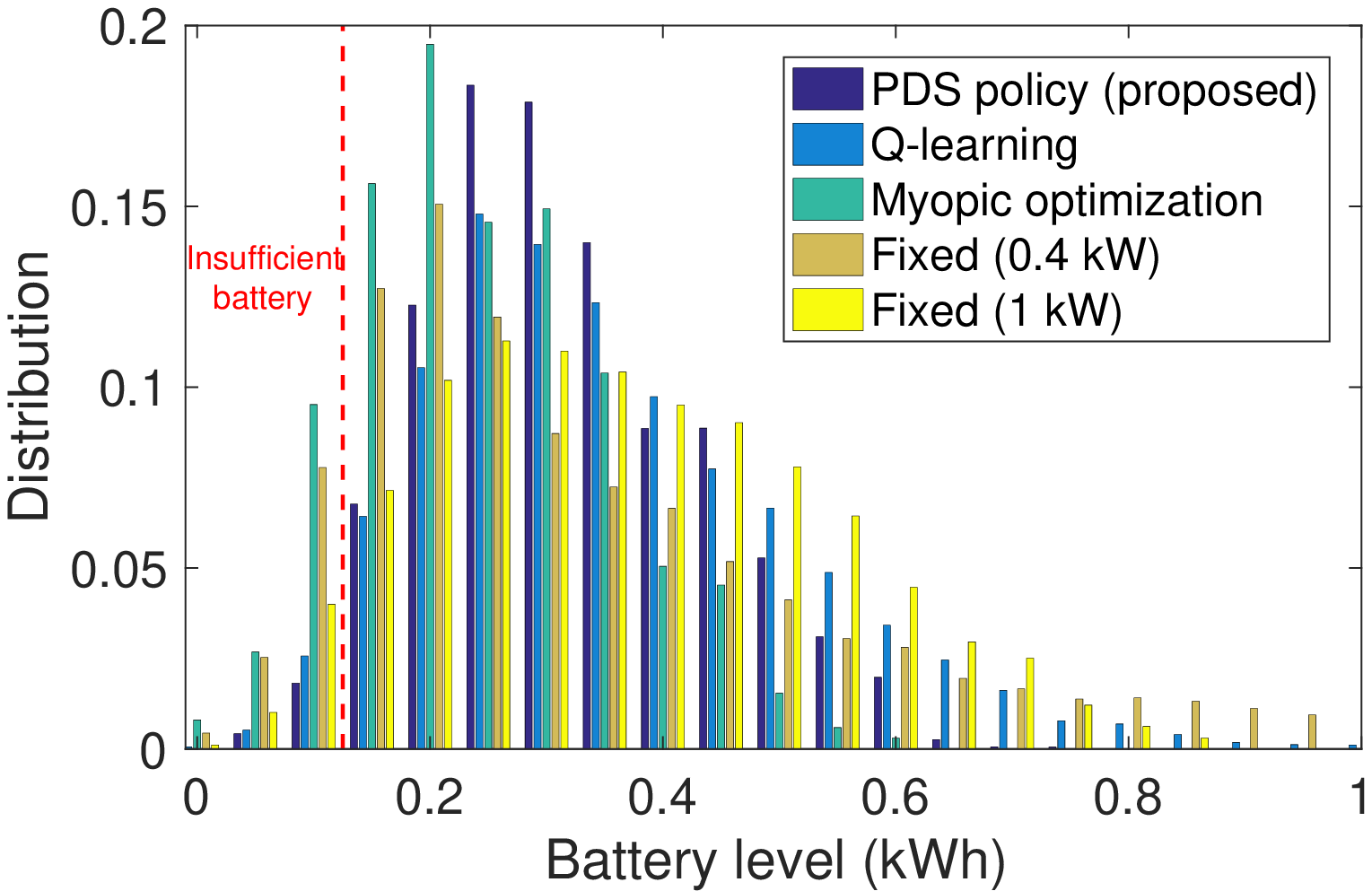}}	
	\subfigure[Battery state distribution (fitted)]{\label{battery_dis_curve}
		\includegraphics[width=0.45\textwidth]{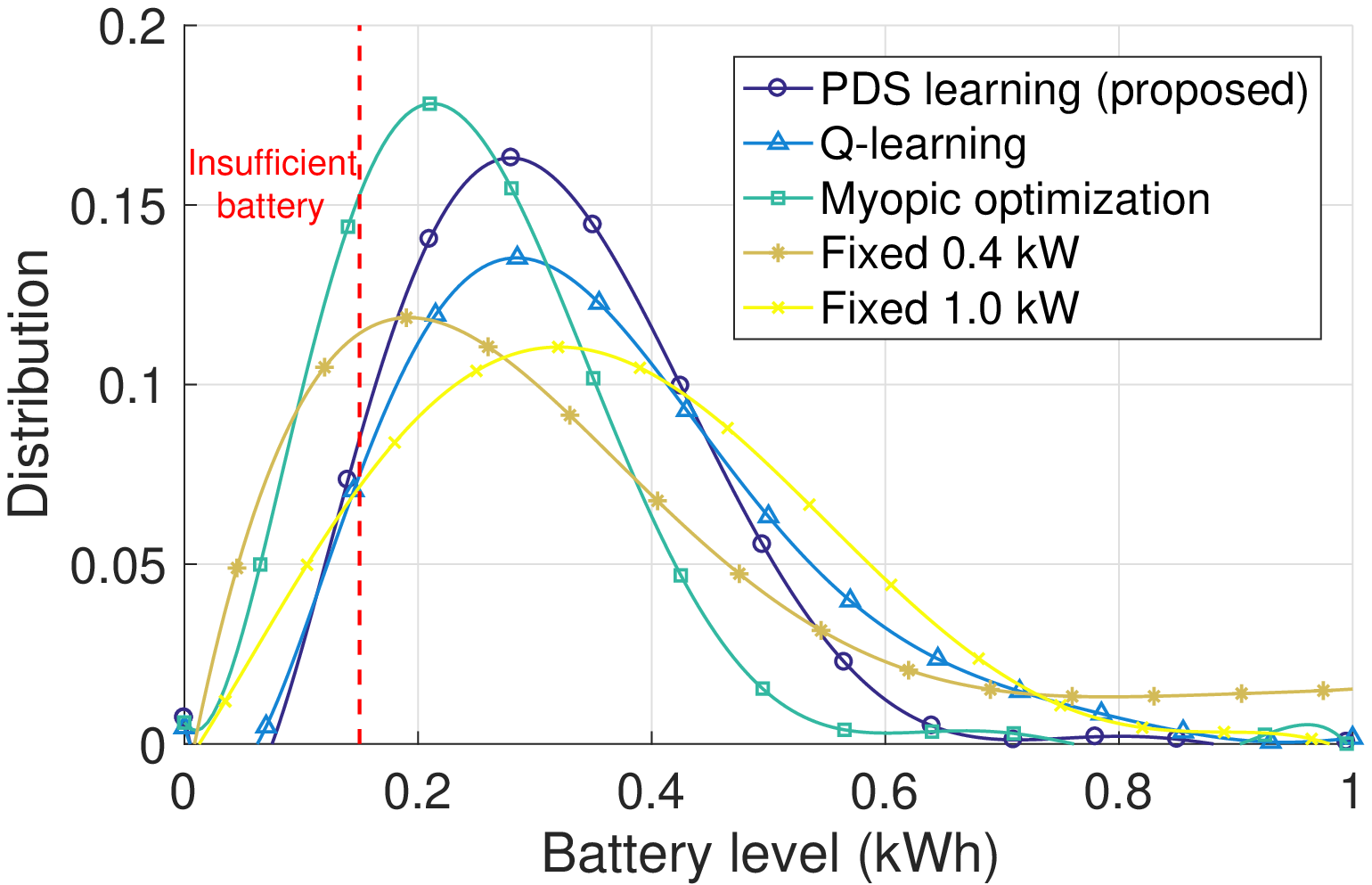}}
	\caption{Battery state distributions}
	\label{battery_dis}
\end{figure}

Figure \ref{battery_dis_bar} shows the distributions of the battery state over the simulated 10000 time slots in one representative simulation run for the various schemes, and Figure \ref{battery_dis_curve} shows the fitted curves (polynomial fit) for better inspection. As can be seen, Myopic optimization results in a large portion of time when the system is in the insufficient battery zone, incurring significant backup power costs. If a too small fixed power demand is used (e.g. 0.4kW), the battery state may spend a considerable amount of time in the high battery level zone (i.e. 0.7 -- 1kWh). This implies that much of the green power cannot be harvested and hence is wasted due to the limited battery capacity constraint. Moreover, using a smaller fixed power for computation does not guarantee that it has a smaller chance to get into the insufficient battery zone. This is because when the battery state is slightly higher than the level that can support basic operation, using a smaller fixed power can easily make battery state drop to the insufficient battery zone in the subsequent time slot whereas if a larger fixed power scheme is used, the system will decide to offload all workload to the cloud without using the local battery power. Although a proper fixed power demand is able to strike a decent balance, it does not well adapt to the changing system dynamics. The proposed PDS-based algorithm achieves the highest harvesting energy efficiency by keeping the battery at a relatively low state while above the insufficient level.

\subsection{Cost Composition}
\begin{figure}
	\centering	
	\subfigure[Run-time cost composition (PDS)]{\label{cost_comp_PDS}
		\includegraphics[width=0.45\textwidth]{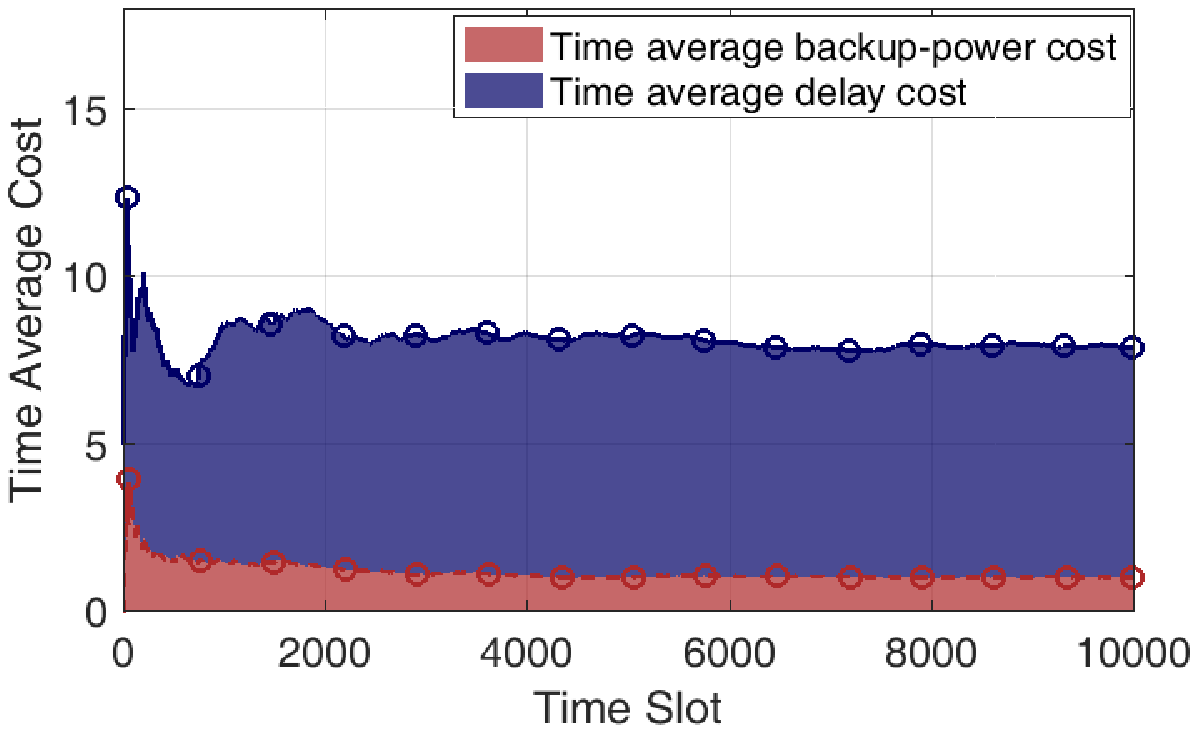}}	
	\subfigure[Run-time cost composition (Myopic)]{\label{cost_comp_myo}
		\includegraphics[width=0.45\textwidth]{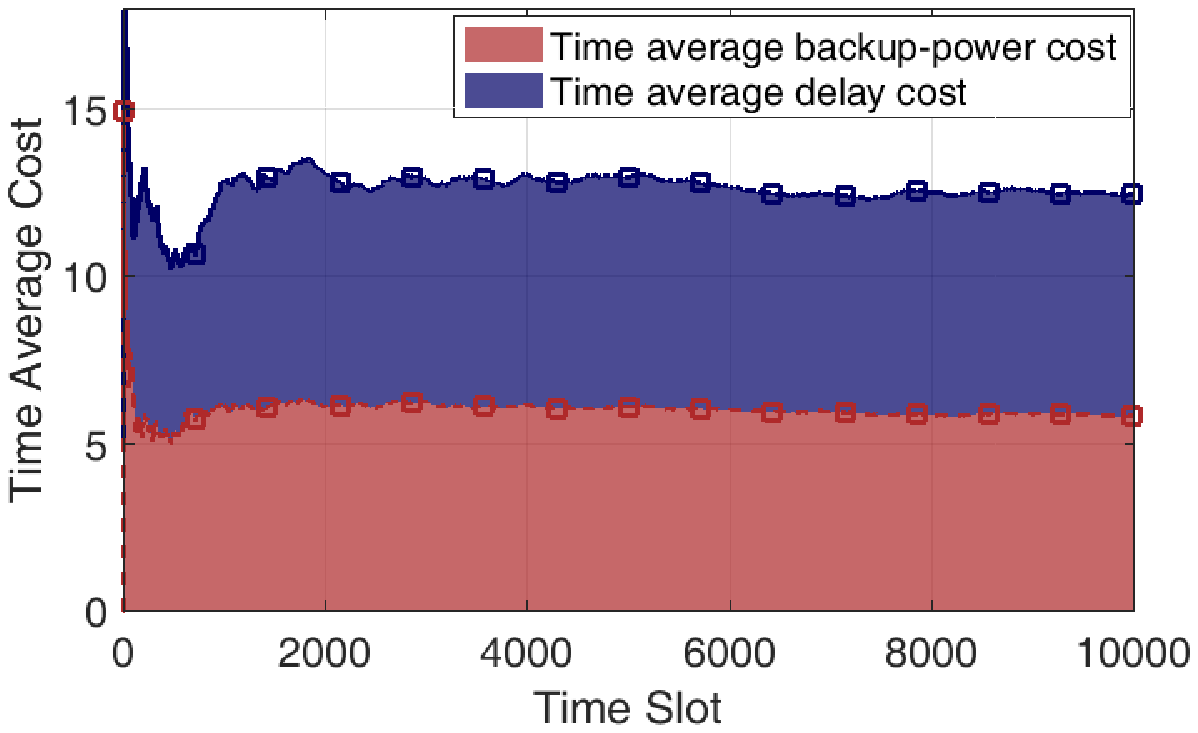}}
	\caption{Run-time cost composition}
	\label{runtime_cost_comp}
\end{figure}

\begin{figure}
	\centering\includegraphics[width=0.45\textwidth]{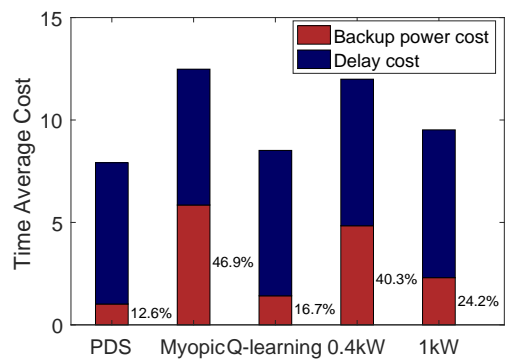}
	\caption{Time-average cost composition ($T=10^4$) }\label{cost_composition}
\end{figure}
Figure \ref{runtime_cost_comp} shows the cost compositions of the PDS-based algorithm and Myopic optimization in 10000 time slots. It can be observed that the proposed PDS-based algorithm significantly cuts the back-up power cost by taking conservative action at low battery states which avoids the usage of backup power. By contrast, Myopic optimization frequently leads the battery state into the insufficient zone as shown in figure \ref{battery_dis} and results in significant backup power costs. Figure \ref{cost_composition} presents the composition of time-average cost at the end of simulation. It can be observed that the PDS-based algorithm and Q-leaning reduce the total cost by considering the future system dynamics and incur low backup power cost accounting for 12.6\% and 16.7\% of the total cost, respectively. These fractions are much lower than those of the remaining schemes which do not consider the long-term system performance.

\subsection{Optimal Offloading Strategy}
\begin{figure}
	\centering	
	\subfigure[Optimal offloading policy ($m=4$)]{\label{mu_m4}
		\includegraphics[width=0.45\textwidth]{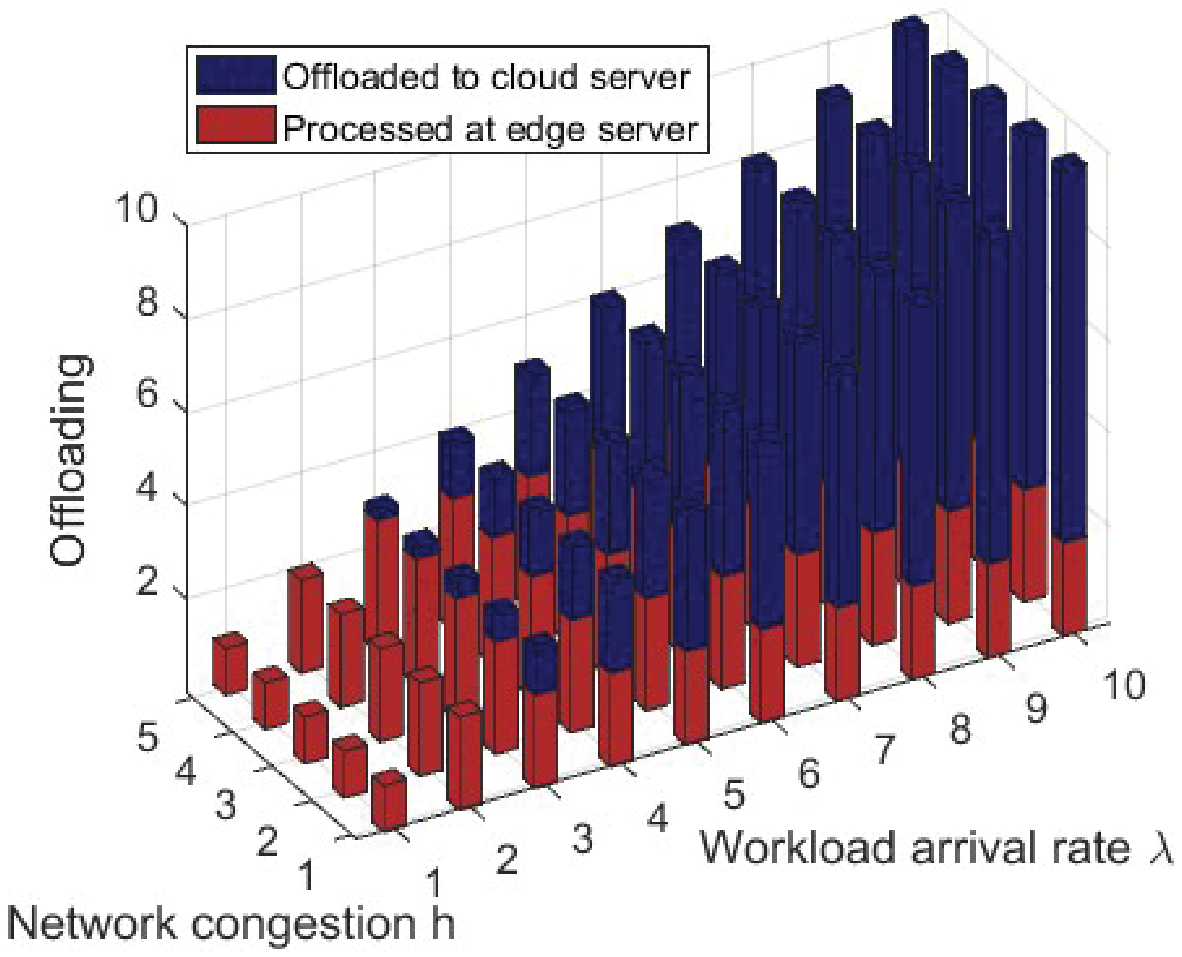}}	
	\subfigure[Optimal offloading policy ($m=10$)]{\label{mu_m10}
		\includegraphics[width=0.45\textwidth]{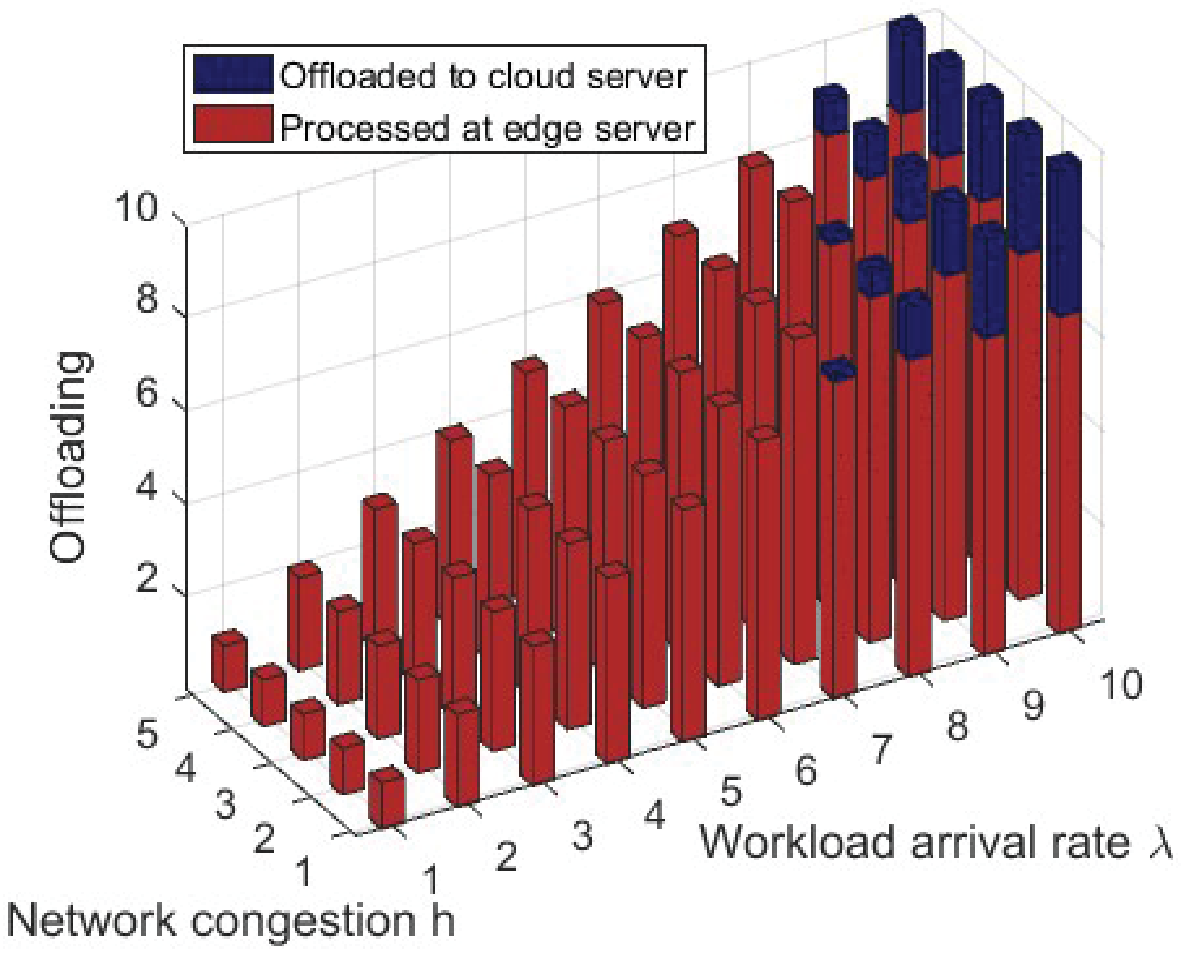}}
	\caption{Optimal offloading policy}
	\label{mu_m}
\end{figure}

Finally, we visualize the optimal offloading strategy in Figure \ref{mu_m} given a fixed number of active servers under various system states. As the network becomes more congested and there are more active edge servers, the optimal strategy chooses to process more workload at the local edge system. However, as the workload arrival rate increases, the amount of workload that can be processed at the local edge system is saturated after certain workload arrival rate.

\section{Conclusion}
In this paper, we studied the joint offloading and autoscaling problem in energy harvesting MEC systems. We found that foresightedness and adaptivity are keys to reliable and efficient operation of renewable-powered MEC. To enable fast learning in the presence of \textit{a priori} unknown system parameters, a PDS-based reinforcement learning algorithm was developed to learn the optimal offloading and autoscaling policy by exploiting the special structure of the considered problem. Our simulations showed that the proposed scheme can significantly improve the edge computing performance even if it is powered by intermittent and unpredictable renewable energy. Future work includes investigating large-scale edge computing systems powered by renewable energy, e.g. green power-aware geographical load balancing.

\bibliographystyle{IEEEtran}
\bibliography{refs}

\begin{thebibliography}{10}
\providecommand{\url}[1]{#1}
\csname url@samestyle\endcsname
\providecommand{\newblock}{\relax}
\providecommand{\bibinfo}[2]{#2}
\providecommand{\BIBentrySTDinterwordspacing}{\spaceskip=0pt\relax}
\providecommand{\BIBentryALTinterwordstretchfactor}{4}
\providecommand{\BIBentryALTinterwordspacing}{\spaceskip=\fontdimen2\font plus
\BIBentryALTinterwordstretchfactor\fontdimen3\font minus
  \fontdimen4\font\relax}
\providecommand{\BIBforeignlanguage}[2]{{%
\expandafter\ifx\csname l@#1\endcsname\relax
\typeout{** WARNING: IEEEtran.bst: No hyphenation pattern has been}%
\typeout{** loaded for the language `#1'. Using the pattern for}%
\typeout{** the default language instead.}%
\else
\language=\csname l@#1\endcsname
\fi
#2}}
\providecommand{\BIBdecl}{\relax}
\BIBdecl

\bibitem{rivera2014gartner}
J.~Rivera and R.~van~der Meulen, ``Gartner says the internet of things will
  transform the data center,'' \emph{Retrieved August}, vol.~5, p. 2014, 2014.

\bibitem{beck2014mobile}
M.~T. Beck and M.~Maier, ``Mobile edge computing: Challenges for future virtual
  network embedding algorithms,'' in \emph{The Eighth International Conference
  on Advanced Engineering Computing and Applications in Sciences (ADVCOMP).
  IARIA}.\hskip 1em plus 0.5em minus 0.4em\relax Citeseer, 2014, pp. 65--70.

\bibitem{vaquero2014finding}
L.~M. Vaquero and L.~Rodero-Merino, ``Finding your way in the fog: Towards a
  comprehensive definition of fog computing,'' \emph{ACM SIGCOMM Computer
  Communication Review}, vol.~44, no.~5, pp. 27--32, 2014.

\bibitem{shi2016edge}
W.~Shi, J.~Cao, Q.~Zhang, Y.~Li, and L.~Xu, ``Edge computing: Vision and
  challenges,'' \emph{IEEE Internet of Things Journal}, vol.~3, no.~5, pp.
  637--646, 2016.

\bibitem{chiang2016fog}
M.~Chiang and T.~Zhang, ``Fog and iot: An overview of research opportunities,''
  \emph{IEEE Internet of Things Journal}, vol.~3, no.~6, pp. 854--864, 2016.

\bibitem{li2015towards}
C.~Li, Y.~Hu, L.~Liu, J.~Gu, M.~Song, X.~Liang, J.~Yuan, and T.~Li, ``Towards
  sustainable in-situ server systems in the big data era,'' in \emph{ACM
  SIGARCH Computer Architecture News}, vol.~43, no.~3.\hskip 1em plus 0.5em
  minus 0.4em\relax ACM, 2015, pp. 14--26.

\bibitem{TaoHan_UNCC_LoadBalancingRAN_HybriedEnergy_ToN}
T.~Han and N.~Ansari, ``Traffic load balancing framework for software-defined
  radio access networks powered by hybrid energy sources,'' \emph{IEEE/ACM
  Transactions on Networking}, vol.~pp, no.~99, March 2015.

\bibitem{sudevalayam2011energy}
S.~Sudevalayam and P.~Kulkarni, ``Energy harvesting sensor nodes: Survey and
  implications,'' \emph{IEEE Communications Surveys \& Tutorials}, vol.~13,
  no.~3, pp. 443--461, 2011.

\bibitem{ulukus2015energy}
S.~Ulukus, A.~Yener, E.~Erkip, O.~Simeone, M.~Zorzi, P.~Grover, and K.~Huang,
  ``Energy harvesting wireless communications: A review of recent advances,''
  \emph{IEEE Journal on Selected Areas in Communications}, vol.~33, no.~3, pp.
  360--381, 2015.

\bibitem{sutton1998reinforcement}
R.~S. Sutton and A.~G. Barto, \emph{Reinforcement learning: An
  introduction}.\hskip 1em plus 0.5em minus 0.4em\relax MIT press, 1998.

\bibitem{patel2014mobile}
M.~Patel, B.~Naughton, C.~Chan, N.~Sprecher, S.~Abeta, A.~Neal \emph{et~al.},
  ``Mobile-edge computing introductory technical white paper,'' \emph{White
  Paper, Mobile-edge Computing (MEC) industry initiative}, 2014.

\bibitem{mao2017mobile}
Y.~Mao, C.~You, J.~Zhang, K.~Huang, and K.~B. Letaief, ``Mobile edge computing:
  Survey and research outlook,'' \emph{arXiv preprint arXiv:1701.01090}, 2017.

\bibitem{chia2014data}
Y.-K. Chia, C.~K. Ho, and S.~Sun, ``Data offloading with renewable energy
  powered base station connected to a microgrid,'' in \emph{Global
  Communications Conference (GLOBECOM), 2014 IEEE}.\hskip 1em plus 0.5em minus
  0.4em\relax IEEE, 2014, pp. 2721--2726.

\bibitem{USC_DynamicBaseStation_Switching_On_Off_TWireless_2013}
E.~Oh, K.~Son, and B.~Krishnamachari, ``Dynamic base station switching-on/off
  strategies for green cellular networks,'' \emph{IEEE Transactions on Wireless
  Communications}, vol.~12, no.~5, pp. 2126--2136, 2013.

\bibitem{Krishnamachari_USC_Energy_efficient_wireless_basestation_CommunicationsMagazine_2011}
E.~Oh, B.~Krishnamachari, X.~Liu, and Z.~Niu, ``Toward dynamic energy-efficient
  operation of cellular network infrastructure,'' \emph{IEEE Communications
  Magazine}, vol.~49, no.~6, pp. 56--61, 2011.

\bibitem{LinWiermanAndrewThereska}
M.~Lin, A.~Wierman, L.~L.~H. Andrew, and E.~Thereska, ``Dynamic right-sizing
  for power-proportional data centers,'' in \emph{IEEE Infocom}, 2011.

\bibitem{ChaoLi_iSwitch_ISCA_2012_Li:2012:ICO:2337159.2337218}
C.~Li, A.~Qouneh, and T.~Li, ``iswitch: Coordinating and optimizing renewable
  energy powered server clusters,'' in \emph{ISCA}, 2012.

\bibitem{Goiri:2011:GSE:2063384.2063411}
I.~Goiri, R.~Beauchea, K.~Le, T.~D. Nguyen, M.~E. Haque, J.~Guitart, J.~Torres,
  and R.~Bianchini, ``Greenslot: scheduling energy consumption in green
  datacenters,'' in \emph{SuperComputing}, 2011.

\bibitem{mao2016dynamic}
Y.~Mao, J.~Zhang, and K.~B. Letaief, ``Dynamic computation offloading for
  mobile-edge computing with energy harvesting devices,'' \emph{IEEE Journal on
  Selected Areas in Communications}, vol.~34, no.~12, pp. 3590--3605, 2016.

\bibitem{neely2010stochastic}
M.~J. Neely, ``Stochastic network optimization with application to
  communication and queueing systems,'' \emph{Synthesis Lectures on
  Communication Networks}, vol.~3, no.~1, pp. 1--211, 2010.

\bibitem{deng2015towards}
R.~Deng, R.~Lu, C.~Lai, and T.~H. Luan, ``Towards power consumption-delay
  tradeoff by workload allocation in cloud-fog computing,'' in
  \emph{Communications (ICC), 2015 IEEE International Conference on}.\hskip 1em
  plus 0.5em minus 0.4em\relax IEEE, 2015, pp. 3909--3914.

\bibitem{xu2016online}
J.~Xu and S.~Ren, ``Online learning for offloading and autoscaling in
  renewable-powered mobile edge computing,'' in \emph{2016 IEEE Global
  Communications Conference (GLOBECOM)}, Dec 2016, pp. 1--6.

\bibitem{GuenterJainWilliams}
B.~Guenter, N.~Jain, and C.~Williams, ``Managing cost, performance and
  reliability tradeoffs for energy-aware server provisioning,'' in \emph{IEEE
  Infocom}, 2011.

\bibitem{zhang2014structure}
Y.~Zhang and M.~van~der Schaar, ``Structure-aware stochastic storage management
  in smart grids,'' \emph{IEEE Journal of Selected Topics in Signal
  Processing}, vol.~8, no.~6, pp. 1098--1110, 2014.

\bibitem{borkar2000ode}
V.~S. Borkar and S.~P. Meyn, ``The ode method for convergence of stochastic
  approximation and reinforcement learning,'' \emph{SIAM Journal on Control and
  Optimization}, vol.~38, no.~2, pp. 447--469, 2000.

\bibitem{bertsekas1995dynamic}
D.~P. Bertsekas, D.~P. Bertsekas, D.~P. Bertsekas, and D.~P. Bertsekas,
  \emph{Dynamic programming and optimal control}.\hskip 1em plus 0.5em minus
  0.4em\relax Athena Scientific Belmont, MA, 1995, vol.~1, no.~2.

\bibitem{puterman2014markov}
M.~L. Puterman, \emph{Markov decision processes: discrete stochastic dynamic
  programming}.\hskip 1em plus 0.5em minus 0.4em\relax John Wiley \& Sons,
  2014.

\end{thebibliography}

\end{document}